\documentclass[11pt,a4paper]{article}
\pdfoutput=1
\usepackage{times,latexsym}
\usepackage{url}
\usepackage[T1]{fontenc}
\usepackage{xspace}
\usepackage{graphicx}
\usepackage{nicefrac}

\usepackage[hyperindex,breaklinks]{hyperref}

\usepackage[acceptedWithA]{tacl2018v2}

\usepackage{xspace,mfirstuc,tabulary}

\newif\iftaclinstructions
\taclinstructionsfalse 
\iftaclinstructions
\renewcommand{\confidential}{}
\renewcommand{\anonsubtext}{(No author info supplied here, for consistency with
TACL-submission anonymization requirements)}
\newcommand{\instr}
\fi

\iftaclpubformat 

\else

\fi

\usepackage{amsmath}
\usepackage{amssymb}
\usepackage{amsfonts}
\usepackage{amsthm}
\usepackage{mathtools}
\usepackage{booktabs}
\usepackage{tabularx}
\usepackage{multirow}
\usepackage{relsize}
\usepackage{enumitem}
\usepackage{makecell}
\usepackage{tikz-dependency}
\usepackage{textcomp}
\usepackage{microtype}
\usepackage{ifmtarg}

\usepackage{algorithm}
\usepackage[noend]{algpseudocode}
\algrenewcommand\algorithmicindent{1.0em}%

\newcommand{\rightcomment}[1]{{\color{gray} \(\triangleright\) {\footnotesize\textit{#1}}}}
\algrenewcommand{\algorithmiccomment}[1]{\hfill \rightcomment{#1}}  
\algnewcommand{\LineComment}[1]{\State \rightcomment{#1}}
\algnewcommand{\LinesComment}[1]{\State \rightcomment{\parbox[t]{\linewidth-\leftmargin-\widthof{\(\triangleright\) }}{#1}}}

\newcommand{\algorithmicfunc}[1]{\textbf{def} #1 :}
\algdef{SE}[FUNC]{Func}{EndFunc}[1]{\algorithmicfunc{#1}}{}
\makeatletter
\ifthenelse{\equal{\ALG@noend}{t}}%
  {\algtext*{EndFunc}}
  {}%
\makeatother

\usepackage{makecell}

\newcommand{\checkNotation}[1]{{#1}}

\newcommand{\defn}[1]{\textbf{#1}}
\renewcommand{\th}{^{\text{th}}}
\renewcommand{\hat}[1]{\widehat{#1}}
\renewcommand{\bar}[1]{\overline{#1}}
\renewcommand{\setminus}{\smallsetminus}

\newcommand{\tree}{\checkNotation{d}}
\newcommand{\st}{\checkNotation{\mathcal{D}}}
\renewcommand{\root}{\checkNotation{\rho}}
\newcommand{\prob}{p}
\newcommand{\qprob}{q}
\newcommand{\R}{\mathbb{R}}
\newcommand{\lab}{\checkNotation{\mathcal{Y}}}
\newcommand{\edges}[0]{\checkNotation{\mathcal{E}}}
\newcommand{\nodes}[0]{\checkNotation{\mathcal{N}}}
\newcommand{\edge}[2]{\checkNotation{(#1\,{\rightarrow}\,#2)}}
\newcommand{\edgeij}[0]{\edge{i}{j}}
\newcommand{\edgekl}[0]{\edge{k}{l}}
\newcommand{\edgeijp}[0]{\edge{i'}{j'}}
\newcommand{\edgeklp}[0]{\edge{k'}{l'}}
\newcommand{\inc}[1]{\checkNotation{\mathcal{D}_{#1}}}
\newcommand{\col}[2]{\checkNotation{\mathcal{L}_{#1#2}}}

\newcommand{\lap}{\checkNotation{\mathbf{L}}}
\newcommand{\laphat}{\checkNotation{\mathbf{\hat{L}}}}

\newcommand{\lapelem}[1]{\checkNotation{\mathrm{L}_{#1}}}
\newcommand{\laphatelem}[1]{\checkNotation{\hat{\mathrm{L}}_{#1}}}

\newcommand{\cached}{\checkNotation{\mathbf{B}}}
\newcommand{\cachedelem}[1]{\checkNotation{\mathrm{B}_{#1}}}
\newcommand{\lapprime}[0]{\checkNotation{\mathrm{L}'}}
\newcommand{\lapp}{\checkNotation{\mathbf{L'}}}
\newcommand{\dlap}[2]{\checkNotation{\lapprime_{#1,#2}}}

\newcommand{\zerovector}{\boldsymbol 0}

\renewcommand{\det}[1]{\left\lvert #1 \right\rvert}

\newcommand{\treefunc}[2]{\checkNotation{#1(#2)}}
\newcommand{\edgefunc}[2]{\checkNotation{#1_{#2}}}
\newcommand{\total}[1]{\checkNotation{\bar{#1}}}
\newcommand{\totalw}[1]{\checkNotation{\widetilde{w_{#1}}}}
\newcommand{\expect}[2]{\checkNotation{\mathbb{E}_{#1}\!\left[ {#2} \right]}}
\newcommand{\wtree}[1]{\treefunc{w}{#1}}
\newcommand{\ptree}[1]{\treefunc{\prob}{#1}}
\newcommand{\qtree}[1]{\treefunc{\qprob}{#1}}
\newcommand{\rtree}[1]{\treefunc{r}{#1}}
\newcommand{\stree}[1]{\treefunc{s}{#1}}
\newcommand{\ttree}[1]{\treefunc{t}{#1}}
\renewcommand{\wedge}[1]{\edgefunc{w}{#1}}
\newcommand{\redge}[1]{\edgefunc{r}{#1}}
\newcommand{\sedge}[1]{\edgefunc{s}{#1}}

\newcommand{\wbar}{\checkNotation{\mathrm{Z}}}
\newcommand{\rbar}{\total{r}}
\newcommand{\sbar}{\total{s}}
\newcommand{\tbar}{\total{t}}

\newcommand{\funcSig}[3]{#1{:}\,\,#2 \mapsto #3}
\newcommand{\treefuncSig}[2]{\funcSig{#1}{\st}{\R^{#2}}}
\newcommand{\otherwbar}[1]{\wbar^{(#1)}}

\newcommand{\rhat}[1]{\checkNotation{\widehat{r_{#1}}}}
\newcommand{\shat}[1]{\checkNotation{\widehat{s_{#1}}}}

\newcommand{\defeq}[0]{\overset{\smaller\mathrm{def}}{=}}
\newcommand{\sumtree}{\sum_{\tree \in \st}}

\newcommand{\pluseq}{\ \texttt{+=}\ }

\newcommand{\bigo}[1]{\mathcal{O}\!\left(#1\right)}
\newcommand{\abs}[1]{\lvert #1 \rvert}

\newcommand{\runtime}[1]{\mathrm{Cost}\!\left\{ #1 \right\}}
\newcommand{\ent}{\mathrm{H}}

\newcommand{\nF}{\checkNotation{F}}
\newcommand{\nN}{\checkNotation{N}}
\newcommand{\nM}{\checkNotation{M}}
\newcommand{\nK}{\checkNotation{K}}
\newcommand{\nQ}{\checkNotation{Q}}
\newcommand{\nR}{\checkNotation{R}}
\newcommand{\nS}{\checkNotation{S}}
\newcommand{\nRs}{\checkNotation{R'}}
\newcommand{\nSs}{\checkNotation{S'}}
\newcommand{\nFs}{\checkNotation{F'}}

\newcommand{\plusequal}{{\,\textsf{+=}\,}}

\newcommand{\real}{\mathbb{R}}

\newcommand{\pwdot}[0]{\!\cdot\!}
\newcommand{\tin}[0]{\,{\in}\,}
\newcommand{\teq}[0]{\,{=}\,}

\newcommand{\monstersum}[1]{{\sum_{\mathclap{\substack{#1}}}}}  

\newcommand{\renyi}{R\'{e}nyi\xspace}

\newcommand{\algFace}[1]{\texttt{#1}}
\newcommand{\algCall}[2]{#1\!\left( #2 \right)}
\newcommand{\algMTT}[0]{\algFace{Z}}

\newcommand{\algFirst}[0]{\algFace{T}_1}
\newcommand{\algSecondVp}[0]{\algFace{T}_2^{\algFace{v}}}
\newcommand{\algSecondHes}[0]{\algFace{T}_2^\algFace{h}}
\newcommand{\algSecond}[0]{\algFace{T}_2}

\newcommand{\target}{\checkNotation{f^*}}

\newcommand{\onehot}[1]{\checkNotation{\overrightarrow{\boldsymbol{1}_{#1}}}}

\usepackage[stable,hang,flushmargin]{footmisc}  

\newcommand{\KL}{\mathrm{KL}}

\newcommand{\ar}{arborescence\xspace}
\newcommand{\ars}{arborescences\xspace}

\newcolumntype{C}{>{\centering\arraybackslash}X}

\newtheorem{thm}{Theorem}

\newtheorem{prop}{Proposition}

\theoremstyle{definition}

\usepackage{cleveref}
\crefname{section}{\S}{\S\S}
\Crefname{section}{\S}{\S\S}
\crefname{table}{Tab.}{}
\crefname{figure}{Fig.}{}
\crefname{algorithm}{Alg}{}
\crefname{algorithm}{Alg}{}
\crefname{line}{Line}{}
\crefname{appendix}{App.}{}
\crefformat{section}{\S#2#1#3}
\crefname{thm}{Theorem}{}
\crefname{prop}{Proposition}{}
\crefname{defin}{Definition}{}
\crefname{lemma}{Lemma}{}
\crefname{cor}{Corollary}{}
\crefname{equation}{}{}

\makeatletter
\newcommand{\toprulealg}{\hrule height.8pt depth0pt \kern2pt} 
\newcommand{\midrulealg}{\kern2pt\hrule\kern2pt} 
\newcommand{\bottomrulealg}{\kern2pt\hrule\relax}
\newcommand{\algcaption}[2][]{%
  \refstepcounter{algorithm}%
  \@ifmtarg{#1}
    {\addcontentsline{loa}{figure}{\protect\numberline{\thealgorithm}{\ignorespaces #2}}}
    {\addcontentsline{loa}{figure}{\protect\numberline{\thealgorithm}{\ignorespaces #1}}}%
  \toprulealg
  \textbf{\fname@algorithm~\thealgorithm}\ #2\par 
  \midrulealg
}
\makeatother

\title{Efficient Computation of Expectations under Spanning Tree Distributions}

\usepackage{stmaryrd}
\usepackage{emoji}

\newcommand{\ucambridge}{\emoji[twitter]{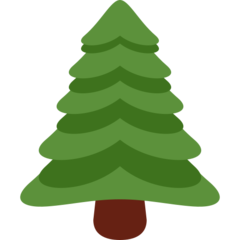}}
\newcommand{\ethz}{\emoji[twitter]{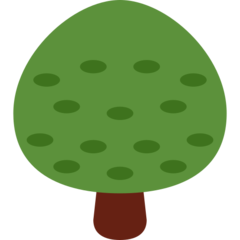}}
\newcommand{\jhu}{\emoji[twitter]{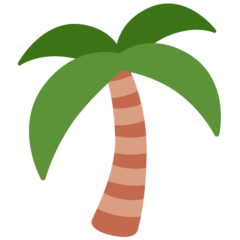}}
\newcommand{\sparkles}{\emoji[twitter]{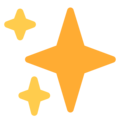}}

\makeatletter
\renewcommand*{\@fnsymbol}[1]{{\hbox{\normalfont\sparkles\ }}}
\makeatother

\author{
{Ran Zmigrod\thanks{\quad Equal contribution}\hspace{7pt} \raise1.0ex\hbox{\normalfont,\ucambridge}\raise1.0ex\hbox{\normalfont}}~\;~Tim Vieira\footnotemark[1]\hspace{10pt}\raise1.0ex\hbox{\normalfont,\jhu}~\;~Ryan Cotterell\raise1.0ex\hbox{\normalfont\ucambridge,\ethz}
\\
  \raise1.0ex\hbox{\normalfont\ucambridge}University of Cambridge~\;~\raise1.0ex\hbox{\normalfont\jhu}Johns Hopkins University~\;~\raise1.0ex\hbox{\normalfont\ethz}ETH Z\"{u}rich \\
  \url{rz279@cam.ac.uk}~\;~\url{tim.f.vieira@gmail.com} \\ \url{ryan.cotterell@inf.ethz.ch}
}
\date{}

\allowdisplaybreaks

\begin{document}
\maketitle

\begin{abstract}
We give a general framework for inference in spanning tree models.
We propose unified algorithms for the important cases of first-order expectations and second-order expectations in edge-factored, non-projective spanning-tree models. 
Our algorithms exploit a fundamental connection between gradients and expectations, which allows us to derive efficient algorithms.
These algorithms are easy to implement with or without automatic differentiation software.
We motivate the development of our framework with several \emph{cautionary tales} of previous research, which has developed numerous inefficient algorithms for computing expectations and their gradients.
We demonstrate how our framework efficiently computes several quantities with known algorithms, including the expected attachment score, entropy, and generalized expectation criteria.  
As a bonus, we give algorithms for quantities that are missing in the literature, including the KL divergence.
In all cases, our approach matches the efficiency of existing algorithms and, in several cases, reduces the runtime complexity by a factor of the sentence length.
We validate the implementation of our framework through runtime experiments.
We find our algorithms are up to $15$ and $9$ times faster than previous algorithms for computing the Shannon entropy and the gradient of the generalized expectation objective, respectively.
\end{abstract}

\section{Introduction}\label{sec:intro}
Dependency trees are a fundamental combinatorial structure in natural language processing.
It follows that probability models over dependency trees are an important object of study.
In terms of graph theory, one can view a (non-projective) dependency tree as an \ar{} (commonly known as a spanning tree) of a graph.  
To build a dependency parser, 
we define a graph where the nodes are the tokens of the sentence, 
and the edges are possible dependency relations between the tokens.
The edge weights are defined by a model, which is learned from data.
In this paper, we focus on edge-factored models 
where the probability of a dependency tree is proportional to the product the weights of its edges.  As there are exponentially many trees in the length of the sentence, 
we require clever algorithms for finding the normalization constant.
Fortunately, the normalization constant for edge-factored models is efficient to compute via to the celebrated matrix--tree theorem.

The matrix--tree theorem \citep{kirchhoff}---more specifically, its counterpart for directed graphs \citep{tutte1984graph}---appeared before the NLP community in an onslaught of contemporaneous papers \citep{koo-et-al-2007, mcdonald-satta-2007,smith-smith-2007} that leverage the classic result to efficiently compute the normalization constant of a distribution over trees.
The result is still used in more recent work \citep{ma,liu-lapata-2018-learning}.
We build upon this tradition through a framework for computing expectations of a rich family of functions under a distribution over trees.
Expectations appear in all aspects of the probabilistic modeling process: training, model validation, and prediction.
Therefore, developing such a framework is key to accelerating progress in probabilistic modeling of trees.

Our framework is motivated by the lack of a unified approach for computing expectations over spanning trees in the literature.
We believe this gap has resulted in the publication of numerous inefficient
algorithms.
We motivate the importance of developing such a framework by highlighting the following \emph{cautionary tales}.\label{cautionary-tales}

\begin{itemize}[leftmargin=1em, topsep=1pt, itemsep=1pt, parsep=1pt]

\item \citet{mcdonald-satta-2007} proposed an inefficient $\bigo{\nN^5}$ algorithm for computing feature expectations, which was much slower than the $\bigo{\nN^3}$ algorithm obtained by \citet{koo-et-al-2007, smith-smith-2007}. The authors subsequently revised their paper.

\item \citet{smith-eisner-2007} proposed an $\bigo{\nN^4}$ algorithm for computing entropy.  Later, \citet{martins-etal-2010-turbo} gave an $\bigo{\nN^3}$ method for entropy, but not its gradient.  Our framework recovers \citet{martins-etal-2010-turbo}'s algorithm, and additionally provides the gradient of entropy in $\bigo{\nN^3}$.
\item \citet{druck-etal-2009-semi} proposed an $\bigo{\nN^5}$ algorithm for evaluating the gradient of the generalized expectation (GE) criterion \citep{mccallum2007generalized}.
The runtime bottleneck of their approach is the evaluation of a covariance matrix, which \citet{druck09covariance} later improved to $\bigo{\nN^4}$.
We show that the gradient of the GE criterion can be evaluated in $\bigo{\nN^3}$.
\end{itemize}

\noindent We summarize our main results below:
\begin{itemize}[leftmargin=1em, topsep=10pt, itemsep=1pt, parsep=1pt]

\item \textbf{Unified Framework}:
We develop an algorithmic framework for calculating expectations over spanning \ars{}.  We give precise mathematical assumptions on the types of functions that are supported.  We provide efficient algorithms that piggyback on automatic differentiation techniques, as our framework is rooted in a deep connection between expectations and gradients \citep{darwiche, li-eisner-2009}.

\item \textbf{Improvements to existing approaches}: 
We give asymptotically faster algorithms where several prior algorithms were known.

\item \textbf{Efficient algorithms for new quantities}: 
We demonstrate how our framework calculates several new quantities, such as the Kullback--Leibler divergence, which (to our knowledge) had no prior algorithm in the literature.

\item \textbf{Practicality}: We present practical speed-ups in the calculation of entropy compared to \citet{smith-eisner-2007}.
We observe speed-ups in the range of $4.1$ and $15.1$ in five languages depending on the typical sentence length.
We also demonstrate a $9$ times speed-up for evaluating the gradient of the GE objective compared to \citet{druck09covariance}.

\item \textbf{Simplicity}: Our algorithms are simple to implement---requiring only a few lines of PyTorch code \citep{pytorch}.
We have released a reference implementation at the following URL
\url{https://github.com/rycolab/tree_expectations}.
\end{itemize}

\section{Distributions over Trees}\label{sec:trees}
We consider the distribution over trees in weighted directed graphs with a designated root node.
A (rooted, weighted, and directed) \defn{graph} is given by $\mathcal{G}=(\nodes,\edges,\root)$.
$\nodes\teq\{1,\dots,\nN\}\cup\{\root\}$ is a set of $\nN{+}1$ nodes where $\root$ is a designated root node.
$\edges$ is a set of weighted edges where
each edge $(i\,{\xrightarrow[]{\wedge{ij}}}\, j) \tin \edges$ 
is a pair of \emph{distinct} nodes such that 
the source node $i \tin \nodes$ points to a 
destination node $j \tin \nodes$ 
with an edge weight $\wedge{ij}\in\real$.
We assume---without loss of generality---that the root node $\root$ has no incoming edges.
Furthermore, we assume only one edge can exist between two nodes.
We consider the multi-graph case
in \cref{sec:dep-parse}.

In natural language processing applications, these weights are typically parametric functions, such as log-linear models \citep{mcdonald-etal-2005-non} or neural networks \citep{dozat, ma}, which are learned from data.

A \defn{tree}\footnote{The more precise graph-theoretic term is \emph{\ar{}}.} $\tree$ of a graph $\mathcal{G}$ is a set of $\nN$ edges such that all non-root nodes $j$ have exactly one incoming edge and the root node $\root$ has at least one outgoing edge.
Furthermore, a tree does not contain any cycles.
We denote the set of all trees in a graph by $\st$ and assume that $\abs{\st}>0$ (this is not necessarily true for all graphs).

The \defn{weight of a tree} $\tree \tin \st$ is defined as:
\begin{equation}
    \wtree{\tree} \defeq \smashoperator{\prod_{\edgeij\in\tree}} \wedge{ij}
\end{equation}
Normalizing the weight of each tree yields a \defn{probability distribution}:
\begin{equation}\label{eq:prob-tree}
    \ptree{\tree} \defeq \frac{\wtree{\tree}}{\wbar}
\end{equation}
where the \defn{normalization constant} is defined as
\begin{equation}
    \wbar \defeq \sumtree \wtree{\tree}=\sumtree \prod_{\edgeij \in \tree} \wedge{ij}
\end{equation}
Of course, for \cref{eq:prob-tree} to be a \emph{proper} distribution, we require $\wedge{ij}\! \ge\! 0$ for all $\edgeij\!\in\!\edges$, and $\wbar\! >\! 0$.

\subsection{The Matrix--Tree Theorem}\label{sec:mtt}
The normalization constant $\wbar$ involves a sum over $\st$, which can grow exponentially large with $\nN$.  
Fortunately, there is sufficient structure in the computation of $\wbar$ that it can be evaluated in $\bigo{\nN^3}$ time. 
The Matrix--Tree Theorem (MTT) \citep{tutte1984graph,kirchhoff} establishes a connection between $\wbar$ and the determinant of the \defn{Laplacian matrix}, $\lap\in\real^{\nN\times \nN}$.
For all $i,j\in\nodes\!\setminus\!\{\root\}$, 
\begin{equation}\label{eq:lap}
    \lapelem{ij} \defeq
    \begin{cases}
    \sum\limits_{i'\in\nodes\setminus\{j\}}\wedge{i'j} & \textbf{ if } i=j \\
    -\wedge{ij} & \textbf{ otherwise}
    \end{cases}
\end{equation}

\newcommand{\foo}[0]{\citet[p. 140]{tutte1984graph}}
\begin{thm}[Matrix--Tree Theorem; \foo]\label{thm:mtt}
For any graph, 
\begin{equation}\label{eq:minor}
    \wbar = \abs{\lap}
\end{equation}
Furthermore, the normalization constant can be computed in $\bigo{\nN^3}$ time.\footnote{For simplicity, we assume that the runtime of matrix determinants is $\bigo{\nN^3}$.  However, we would be remiss if we did not mention that algorithms exist to compute the determinant more efficiently \citep{dumas2016fast}.}
\end{thm}

\subsection{Dependency parsing \& the Laplacian zoo}
\label{sec:dep-parse}
Graph-based dependency parsing can be encoded as follows.
For each sentence of length $\nN$,
we create a graph $\mathcal{G}=(\nodes,\edges,\root)$ where 
each non-root node represents a token of the sentence, and $\root$ represents a special root symbol of the sentence.
Each edge $\edgeij$ in the graph represents a \emph{possible} dependency relation between 
head word $i$ and modifier word $j$.
\cref{fig:deptree} gives an example dependency tree.
In the remainder of this section, 
we give several variations on the Laplacian matrix that encode different sets of valid trees.\footnote{The reader may want to skip this section on their first reading.}

In many cases of dependency parsing, we want $\root$ to have exactly one outgoing edge.
This is motivated by linguistic theory, where the root of a sentence should be a token in the sentence rather than a special root symbol \citep{tesniere1959elements}.
There are exceptions to this, such as 
parsing Twitter \citep{kong-etal-2014-dependency}
and parsing specific languages (e.g., The Prague Treebank \citep{prague_dep}).
We call these \defn{multi-root trees}\footnote{We follow the conventions of \citet{koo-et-al-2007} and say ``single-root'' and ``multi-root'' when we \emph{technically} mean the number of outgoing edges from the root $\root$, and \emph{not} the number of root nodes in a tree, which is always one.} and are represented by the set $\st$ as described earlier.
Therefore, the normalization constant over all multi-root trees can be computed by a direct application of \cref{thm:mtt}.

However, in most dependency parsing corpora, 
only one edge may emanate from the root \citep{ud,zmigrod2020mind}.
Thus, we consider the set of \defn{single-rooted trees}, denoted $\st^{(1)}$.
\citet{koo-et-al-2007} adapts \cref{thm:mtt} to efficiently compute $\wbar$ for the set $\st^{(1)}$ with the \defn{root-weighted Laplacian},\footnote{The choice to replace row 1 by the root edges is done by convention, we can replace \emph{any} row in the construction of $\laphat$.} $\laphat\in\real^{\nN\times \nN}$
\begin{equation}\label{eq:root-weighted-laplacian}
    \laphatelem{ij}=
    \begin{cases}
    \wedge{\root j} & \textbf{ if } i=1 \\
    \sum\limits_{i'\in\nodes\setminus\{\root,j\}}\wedge{i'j} & \textbf{ if } i=j \\
    -\wedge{ij} & \textbf{ otherwise}
    \end{cases}
\end{equation}

\begin{prop}\label{prop:mtt-root-weighted-laplacian}
For any graph, the normalization constant over all single-rooted trees is given by the determinant of the root-weighted Laplacian \cite[Prop. 1]{koo-et-al-2007}
\begin{equation}
    \wbar=\abs{\laphat}
\end{equation}
Furthermore, the normalization constant for single-rooted trees can be computed in $\bigo{\nN^3}$ time.
\end{prop}

\paragraph{Labeled trees.}
To encode \emph{labeled} dependency relations in our set of trees, we simply augment edges with labels---resulting in a \defn{multi-graph} in which multiple edges may exist between pairs of nodes.
Now, edges take the form $(i\!\xrightarrow[]{y / w_{ijy}}\!j)$ where $i$ and $j$ are the source and destination nodes as before, $y \in \lab$ is the label, and $w_{ijy}$ is their weight.
\begin{prop}\label{prop:labeled-mtt}
For any multi-graph, the normalization constant for multi-root or single-rooted trees
can be calculated using \cref{thm:mtt} or \cref{prop:mtt-root-weighted-laplacian} (respectively) 
with the edge weights,
\begin{equation}
    \wedge{ij}=\sum_{y\in\lab}\wedge{ijy}
\end{equation}
Furthermore, the normalization constant can be computed in $\bigo{\nN^3 + \abs{\lab} \nN^2}$ time.\footnote{The algorithms given in later sections will not provide full details for the labeled case due to space constraints, but we assure the reader that our algorithms can be straightforwardly generalized to the labeled setting.}
\end{prop}

\begin{figure}[t]
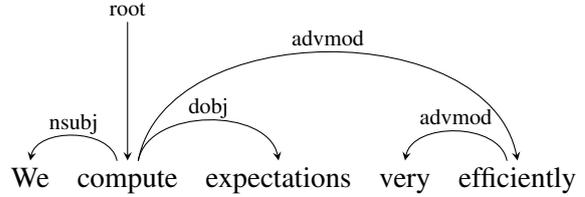

    \centering
    \begin{dependency}[theme = simple]
    \begin{deptext}[column sep=0.5em]
      We \& compute \& expectations \& very \& efficiently \\
    \end{deptext}
    \deproot{2}{\large root}
    \depedge{2}{1}{\large nsubj}
    \depedge{2}{3}{\large dobj}
    \depedge{2}{5}{\large advmod}
    \depedge{5}{4}{\large advmod}
    \end{dependency}
    \caption{Example of a dependency tree}
    \label{fig:deptree}
\end{figure}

\paragraph{Summary.} We give common settings in which the MTT can be adapted to efficiently compute $\wbar$ for different sets of trees.
The choice is dependent upon the task of interest, and one must be careful to choose the correct Laplacian configuration.
The results we present in this paper are modular in the specific choice of Laplacian.
For the remainder of this paper, we assume the unlabeled tree setting and will refer to the set of trees as simply $\st$ and our choice of Laplacian as $\lap$.

\section{Expectations}
\label{sec:expect}

In this section, we characterize the family of expectations that our framework supports.
Our framework is an extension of \citet{li-eisner-2009} to distributions over spanning trees.
In contrast, their framework considers expectations over distributions that can be factored as B-hypergraphs \citep{gallo}.
Our distributions over trees cannot be cast as polynomial-size B-hypergraphs.
Another important distinction between our framework and that of \citet{li-eisner-2009} is that we do not use the semiring abstraction as it is algebraically too weak to compute the determinant efficiently.\footnote{\label{fn:generalized-mtt}In fact, \citet{Jerrum-Snir-1982} proved that the
partition function for spanning trees requires an exponential number of additions and multiplications in the semiring model of computation (i.e., assuming that subtraction is not allowed).
Interestingly, division is not required, but algorithms for division-free determinant computation run in $\bigo{\nN^4}$ \citep{kaltofen1992computing}. An excellent overview of \emph{the power of subtraction} in the context of dynamic programming is given in \citet[Ch.\@ 3]{miklos2019computational}.  It would appear as if commutative rings would make a good level of abstraction as they admit efficient determinant computation. 
Interestingly, this means that we cannot use the MTT in the max-product semiring to (efficiently) find the maximum weight tree because max does not have an inverse.  Fortunately, there exist $\bigo{\nN^2}$ algorithms to find the maximum weight tree for both the single-root and multi-root settings \citep{zmigrod2020mind,GabowT84}.%
}

The \defn{expected value} of a function $\treefuncSig{f}{\nF}$ is defined as follows
\begin{equation}
    \expect{\tree}{f(\tree)} \defeq \sumtree\ptree{\tree}\treefunc{f}{\tree}  \label{eq:general-expectation}
\end{equation}
Without any assumptions on $f$, computing \cref{eq:general-expectation} is intractable.\footnote{Of course, one could use sampling methods, such as Monte Carlo, to approximate \cref{eq:general-expectation}.  Sampling methods may be efficient if the variance of $f$ under $p$ is not too large.} 
In the remainder of this section, we will characterize a class of functions $f$ whose expectations can be efficiently computed.

The first type of functions we consider are functions that are \defn{additively decomposable}
along the edges of the tree.
Formally, a function $\treefuncSig{r}{\nR}$ is additively decomposable if it can be written as
\begin{equation}
    \rtree{\tree} = \smashoperator{\sum_{\edgeij\in\tree}}\redge{ij}
\end{equation}
where we abuse notation slightly by for any function $\treefuncSig{r}{\nR}$, we consider the edge function $\redge{ij}$ as a vector of edge values.
An example of an additively decomposable function is  $\rtree{\tree} = -\log \ptree{\tree}$ whose expectation gives the Shannon entropy.\footnote{
Proof: $-\log \ptree{\tree} \teq -\log(\frac{1}{\wbar} \prod_{\edgeij\in\tree} \wedge{ij})$\\
$\teq \log\wbar \,{-}\, \sum_{\edgeij \in \tree} \log\wedge{ij}$.\\ $\Rightarrow$ $\redge{ij} \teq \frac{1}{\nN}\log\wbar\,{-}\,\log\wedge{ij}$.}
Other first-order expectations include the expected attachment score and the Kullback--Leibler divergence.
We demonstrate how to compute these in our framework in and \cref{sec:uas} and \cref{sec:kl}, respectively.

The second type of functions we consider are functions that are \defn{second-order additively decomposable} along the edges of the tree.
Formally, a function $\treefuncSig{r}{\nR}$ is second-order additively decomposable if it can be written as the outer product of two additively decomposable functions, $\treefuncSig{r}{\nR}$ and  $\treefuncSig{s}{\nS}$
\begin{equation}\label{eq:first-order}
    \ttree{\tree} = \rtree{\tree} \stree{\tree}^\top
\end{equation}
Thus, $\ttree{\tree} \in \real^{\nR \times \nS}$ is generally a matrix.

An example of such a function is the gradient of entropy (see \cref{sec:ent}) or the GE objective \citep{mccallum2007generalized} (see \cref{sec:ge})
with respect to the edge weights.
Another example of a second-order additively decomposable function is the covariance matrix.
Given two feature functions $\treefuncSig{r}{\nR}$ and $\treefuncSig{s}{\nS}$,
their covariance matrix 
is $\expect{\tree}{\rtree{\tree}\stree{\tree}^\top} - \expect{\tree}{\rtree{\tree}}\expect{\tree}{\stree{\tree}}^\top$.  Thus, it is second-order additively decomposable function as long as $\rtree{d}$ and $\stree{d}$ are additively decomposable.

One family of functions which can be computed efficiently but we will not explore here are those who are \defn{multiplicatively decomposable} over the edges.
A function $\treefuncSig{q}{\nQ}$ is multiplicatively decomposable if it can be written as
\begin{equation}
    \treefunc{q}{\tree} = \smashoperator{\prod_{\edgeij\in\tree}}\edgefunc{q}{ij}
\end{equation}
where the product of $q_{ij}$ is an element-wise vector product.
These functions form a family that we will call zero$\th$-order expectations and can be computed with a constant number of calls to MTT (usually two or three).
Examples of these include the \renyi entropy and $\ell_p$-norms.\footnote{
The $\ell_k$ norm of the distribution $p$
often denoted as
$\| p \|_k \defeq \left( \sum_{\tree\in\st}p(\tree)^{k} \right)^{1/k}$ for $k\!\ge\!0$.  
It is computable from a zero$\th$-order expectation 
because it can be written as
$(\frac{\otherwbar{k}}{\wbar^{k}})^{1/k}$ 
where $\otherwbar{k} \teq \sum_{\tree \in \st}\wtree{\tree}^{k}
\!=\! \sum_{\edgeij\in \tree}\wedge{ij}^{\!\!\!\!k}$, 
which is clearly a zero$\th$-order expectation.
Similarly, the \renyi entropy of order $\alpha\!\ge\!0$ with $\alpha \!\ne\! 1$ is
$\ent_\alpha(p)\,{\defeq}\, \frac{1}{1-\alpha}\log\left(\sum_{\tree\in\st}p(\tree)^{\alpha}\right) =$ 
$\frac{1}{1-\alpha}\log\left(\frac{\otherwbar{\alpha}}{\wbar^{\alpha}}\right)$.
}

\section{Connecting gradients and expectations}\label{sec:connections}
In this section, we build upon a fundamental connection between gradients and expectations \citep{darwiche,li-eisner-2009}.
This connection allows us to build on work in automatic differentiation to obtain efficient gradient algorithms.
While the propositions in this section are inspired from past work, we believe that the presentation and proofs of these propositions have previously not been clearly presented.\footnote{\citet[Section 5.1]{li-eisner-2009} provides a similar derivation to \cref{prop:first-weight} and \cref{prop:rbar} for hypergraphs.}
We find it convenient to work with unnormalized expectations, or totals (for short).
We denote the \defn{total} of a function $f$ as $\total{f} \defeq \sumtree \wtree{\tree} \treefunc{f}{\tree}$. 
We recover the expectation with $\expect{\prob}{f} = \total{f} / \wbar$.
We note that totals (on their own) may be of interest in some applications \citep[Section 5.3]{vieira-eisner-2017-learning}.

\paragraph{The first-order case.}
Specifically, the partial derivative $\frac{\partial \wbar}{\partial \wedge{ij}}$ is useful for determining the \defn{total weight} of trees which include the edge $\edgeij$,
\begin{equation}
    \totalw{ij} \defeq \sum_{\tree\in\inc{ij}}\wtree{\tree}
\end{equation}
where $\inc{ij}\defeq \{ \tree \in \st \mid \edgeij \in \tree \}$.
Furthermore, $\prob(\edgeij \in \tree) = \totalw{ij}/\wbar = \frac{\wedge{ij}}{\wbar} \frac{\partial \wbar}{\partial \wedge{ij}}$.\footnote{Some authors (e.g., \citet{wainwright}) prefer to work with an exponentiated representation $\wedge{ij} = \exp(\theta_{ij})$ so that $\nabla_{\theta_{ij}} \log \wbar = p( \edge{i}{j} \in \tree )$.  This avoids an explicit division by $\wbar$, and multiplication by $\wedge{ij}$ as these operations happens by virtue of the chain rule.}
\begin{prop}
\label{prop:first-weight} 
For any edge $\edgeij$,
\begin{equation}\label{eq:total}
    \totalw{ij}=\frac{\partial \wbar}{\partial \wedge{ij}}\wedge{ij}
\end{equation}
\end{prop}
\begin{proof}
\begin{align*}
    \totalw{ij} &= \sum_{\tree\in\inc{ij}}\wtree{\tree} \\
    &= \sum_{\tree\in\inc{ij}}\prod_{\edgeijp\in\tree}\wedge{i'j'} \\
    &= \wedge{ij} \sum_{\tree\in\inc{ij}}\prod_{\substack{\edgeijp\in \\\tree\setminus \{\edgeij\}}}\wedge{i'j'}  \\
    &=\wedge{ij} \frac{\partial}{\partial \wedge{ij}} \sumtree\prod_{\edgeijp\in\tree} \wedge{i'j'} \\
    &= \frac{\partial \wbar}{\partial \wedge{ij}} \wedge{ij}
\end{align*}
\end{proof}

\cref{prop:rbar} will establish a connection between the unnormalized expectation $\rbar$ and $\nabla\wbar$.

\begin{prop}
\label{prop:rbar}
For any additively decomposable function $\treefuncSig{r}{\nR}$, the total $\rbar$ can be computed using a gradient--vector product
\begin{equation}
    \rbar = \smashoperator{\sum_{\edgeij\in\edges}}\totalw{ij}\redge{ij}
\end{equation}
\end{prop}
\begin{proof}
\begin{align*}
    \rbar &= \sumtree\wtree{\tree}\rtree{\tree} \\
    &= \sumtree\wtree{\tree}\sum_{\edgeij\in \tree} \redge{ij} \\
    &= \sumtree\sum_{\edgeij\in \tree}\wtree{\tree}\redge{ij} \\
    &= \sum_{\edgeij\in\edges}\sum_{\tree\in\inc{ij}}\wtree{\tree}\redge{ij} \\
    &= \smashoperator{\sum_{\edgeij\in\edges}}\totalw{ij}\redge{ij}
\end{align*}
\end{proof}

\paragraph{The second-order case.}
We can similarly use $\frac{\partial^2 \wbar}{\partial \wedge{ij}\, \partial \wedge{kl}}$ to determine the total weight of trees which include both $\edgeij$ and $\edgekl$ with $\edgeij \ne \edgekl$\footnote{As each edge can only appear once in a tree, $\totalw{ij,ij}=0$.}
\begin{equation}
    \totalw{ij,kl}\defeq\sum_{\tree\in\inc{ij,kl}}\wtree{\tree}
\end{equation}
where $\inc{ij,kl} \!\defeq\! \{ \tree \!\in\! \st \mid \edgeij \!\in\! \tree, \edgekl \!\in\! \tree \}$.
Furthermore, $\frac{\totalw{ij,kl}}{\wbar} \!=\! p( \edgeij \!\in\! \tree, \edgekl \!\in\! \tree)$.

\begin{prop}
\label{prop:second-weight}
For any pair of edges $\edgeij$ and $\edgekl$ such that $\edgeij\neq\edgekl$,
\begin{equation}\label{eq:total2}
    \totalw{ij,kl}=\frac{\partial^2 \wbar}{\partial \wedge{ij}\,\partial\wedge{kl}}\wedge{ij}\wedge{kl}
\end{equation}
\end{prop}
\begin{proof}
\begin{align*}
    \totalw{ij,kl} &= \sum_{\tree\in\inc{ij,kl}}\wtree{\tree} \\
    &= \sum_{\tree\in\inc{ij,kl}}\prod_{\edgeklp\in\tree}\wedge{k'l'} \\
    &= \wedge{ij}\wedge{kl}\frac{\partial^2}{\partial\wedge{ij}\partial\wedge{kl}} \sumtree\prod_{\edgeijp\in\tree}\wedge{i'j'} \\
    &= \frac{\partial^2 \wbar}{\partial \wedge{ij}\,\partial\wedge{kl}}\wedge{ij}\wedge{kl}
\end{align*}
\end{proof}%

\cref{prop:drbar} will relate $\nabla^2\wbar$ to $\nabla\rbar$.
This will be used in \cref{prop:tbar} to establish a connection between the total $\tbar$ and $\nabla^2\wbar$, and additionally establishes a connection between $\tbar$ and $\nabla\rbar$.

\begin{prop}
\label{prop:drbar}
For any additively decomposable function $\treefuncSig{r}{\nR}$
that does not depend on $w$,\footnote{More precisely, $\frac{\partial \rtree{\tree}}{\partial \wedge{ij}} = \zerovector$ for all $\tree\in\st$ and $\edgeij\in\edges$.} and edge $\edgeij\in\edges$,
\begin{align}\label{eq:drbar}
    \wedge{ij}&\frac{\partial\rbar}{\partial \wedge{ij}} = \totalw{ij}\redge{ij} + \smashoperator{\sum_{\edgekl\in\edges}}\totalw{ij,kl}\redge{kl}
\end{align}
\end{prop}
\begin{proof}
\begin{align*}
    &\wedge{ij}\frac{\partial\rbar}{\partial \wedge{ij}} \\
    &= \wedge{ij}\frac{\partial}{\partial\wedge{ij}}\left(\sum_{\edgekl\in\edges}\frac{\partial\wbar}{\partial\wedge{kl}}\wedge{kl}\redge{kl}\right) \\
    &= \wedge{ij}\frac{\partial\wbar}{\partial\wedge{ij}}\redge{ij} + \wedge{ij}\smashoperator{\sum_{\edgekl\in\edges}}\ \ \ \frac{\partial^2\wbar}{\partial\wedge{ij}\partial\wedge{kl}}\wedge{kl}\redge{kl} \\ 
    &= \totalw{ij}\redge{ij} + \smashoperator{\sum_{\edgekl\in\edges}}\totalw{ij,kl}\redge{kl}
\end{align*}
\end{proof}

\begin{prop}
\label{prop:tbar}
For any second-order additively decomposable function $\treefuncSig{t}{{R \times S}}$, which is expressed as the outer product of additively decomposable functions,
$\treefuncSig{r}{\nR}$ and
$\treefuncSig{s}{\nS}$, 
$\ttree{\tree} = \rtree{\tree} \stree{\tree}^{\top}$, where $r$ does not depend on $w$,
the total $\tbar$ can be computed using a Jacobian--matrix product
\begin{equation}\label{eq:tbar-r}
    \tbar = \sum_{\edgeij\in\edges}\frac{\partial \rbar}{\partial \wedge{ij}}\wedge{ij}\sedge{ij}^{\top}
\end{equation}
or a Hessian--matrix product
\begin{align}
    \tbar = \smashoperator{\sum_{\edgeij\in\edges}}\totalw{ij}\redge{ij}\sedge{ij}^{\top} + \smashoperator{\sum_{\edgekl\in\edges}}\totalw{ij,kl}\redge{ij}\sedge{kl}^{\top}
    \label{eq:tbar-Z}
\end{align}
\end{prop}
\begin{proof}
We first prove \cref{eq:tbar-r}
\begin{align*}
    & \tbar \\
    &= \sum_{\tree\in\st}\wtree{\tree}\rtree{\tree}\stree{\tree}^{\top} \\
    &= \sum_{\tree\in\st}\wtree{\tree}\rtree{\tree}\smashoperator{\sum_{\edgeij\in \tree}}\sedge{ij}^{\top} \\
    &= \sum_{\tree\in\st}\sum_{\edgeij\in \tree}\wtree{\tree}\rtree{\tree}\sedge{ij}^{\top} \\
    &= \sum_{\edgeij\in\edges}\sum_{\tree\in\inc{ij}}\wtree{\tree}\rtree{\tree}\sedge{ij}^{\top} \\
    &= \sum_{\edgeij\in\edges} \wedge{ij}\frac{\partial}{\partial \wedge{ij}} \!\left(  \sum_{\tree\in\st}\wtree{\tree}\rtree{\tree}\!\right)\!\sedge{ij}^{\top} \\
    &= \sum_{\edgeij\in\edges}\wedge{ij}\frac{\partial \rbar}{\partial \wedge{ij}}\sedge{ij}^{\top}
\end{align*}
Then \cref{eq:tbar-Z} immediately follows by substituting \cref{eq:drbar} into  \cref{eq:tbar-r} and expanding the summation.
\end{proof}

\paragraph{Remark.} There is a simple recipe to compute
${\nabla \rbar_n}$ for each $n=1,\dots,\nR$.
First, some notation; let $\onehot{ij}$ be a vector over $\edges$ with a $1$ in dimension $\edgeij$, and zeros elsewhere.
By plugging $[\redge{ij}]_n$ and  $\sedge{ij}\!=\frac{1}{\wedge{ij}}\onehot{ij}$ into \cref{eq:tbar-r}, 
we can compute $\tbar_n = {\nabla \rbar_n}$.\footnote{Note that when $\wedge{ij}=0$, we can set $\sedge{ij}=\zerovector$.}
However, if $r$ depends on $w$, we must add the following first-order term, which is due to the product rule
\begin{equation}\label{eq:nablabar}
    \nabla \rbar_n = \tbar_{n} + \underbrace{\smashoperator{\sum_{\edgeij\in\edges}}\totalw{ij}\nabla[\redge{ij}]_n}_{\text{first-order term}}
\end{equation}
We provide the details for computing the gradients of two first-order quantities, Shannon Entropy and the KL divergence, 
using this recipe in \cref{sec:ent} and \cref{sec:kl}, respectively.

\section{Algorithms}\label{sec:algs}
Having reduced the computation of $\rbar$ and $\tbar$ to finding derivatives of $\wbar$ in \cref{sec:connections}, we now describe efficient algorithms that exploit this connection.
The main algorithmic ideas used in this section are based on automatic differentiation (AD) techniques \citep{griewank-walther}.
These are general-purpose techniques for efficiently evaluating gradients given algorithms that evaluate the functions.
In our setting, the algorithm in question is an efficient procedure for evaluating $\wbar$, such as the procedure we described in \cref{sec:mtt}.
While we provide derivatives \cref{sec:derivatives-of-Z} in our algorithms, these can also be evaluated using any AD library, such as JAX \citep{jax}, PyTorch (our choice) \citep{pytorch}, or TensorFlow \citep{tensorflow}.

\cref{prop:rbar} is realized as $\algFirst$ in \cref{fig:alg} and \cref{eq:tbar-r} and \cref{eq:tbar-Z} are realized as $\algSecondVp$ and $\algSecondHes$ in \cref{fig:second} respectively.
We provide the runtime complexity of each step in the algorithms. These will be discussed in more detail in \cref{sec:runtime}.

\subsection{Derivatives of $\wbar$}
\label{sec:derivatives-of-Z}
All three algorithms rely on first- or second-order derivatives of $\wbar$.
Since $\wbar = \det{\lap}$, we can express its gradient via Jacobi's formula and an application of the chain rule\footnote{The derivative of $\det{\lap}$ can also be given using the matrix adjugate, $\nabla \wbar = \mathrm{adj}(\lap)^{\top}$.
There are benefits to using the adjugate as it is more numerically stable and equally efficient \citep{stewart1998adjugate}.
In fact, any algorithm that computes the determinant can be algorithmically differentiated to obtain an algorithm for the adjugate.}

\begin{equation}\label{eq:dz}
    \frac{\partial \wbar}{\partial \wedge{ij}} = \wbar \smashoperator{\sum_{(i',j')\in\col{i}{j}}}\cachedelem{i'j'}\dlap{i'j'}{ij}
\end{equation}
where
\begin{equation}\label{eq:cached}
    \cached=\lap^{-\top}
\end{equation}
is the transpose of $\lap^{-1}$, $\dlap{i'j'}{ij}=\frac{\partial\lapelem{i'j'}}{\partial\wedge{ij}}$, and $\col{i}{j}$ is the set of pairs where $(i', j')\in\col{i}{j}$ means that $\dlap{i'j'}{ij}\neq 0$.
We define $\cachedelem{\root j'} \defeq 0$ for any $j' \in \nodes$.
\citet{koo-et-al-2007} show that for any $i$ and $j$, $\abs{\col{i}{j}}\le 2$ in the unlabeled case, indeed, $\dlap{i'j'}{ij}$ is given by
\begin{equation}\label{eq:dlap}
    \dlap{i'j'}{ij}=\begin{cases}
    1 & \textbf{ if } i'\in\{1,j\}, j'=j \\
    -1 & \textbf{ if } i'=i,j'=j \\
    0 & \textbf{ otherwise}
    \end{cases}
\end{equation}
Their result holds for any Laplacian encoding we gave in \cref{sec:dep-parse}.\footnote{We have that $\abs{\col{i}{j}}\le 2\abs{\lab}$ in the labeled case.}

The second derivative of $\wbar$ can be evaluated as follows\footnote{We provide a derivation in \cref{app:hes}. \citet{druck09covariance} give a similar derivation for the Hessian, which we have generalized to any second-order quantity.}
\begin{equation}\label{eq:wbar-hes}
    \frac{\partial^2 \wbar}{\partial\wedge{ij}\partial\wedge{kl}} = \smashoperator{\sum_{\substack{(i', j')\in\col{i}{j} \\ (k', l')\in\col{k}{l}}}} \dlap{i'j'}{ij}\frac{\partial^2 \wbar}{\partial\lapelem{i'j'}\partial\lapelem{k'l'}}\dlap{k'l'}{kl}
\end{equation}
where
\begin{equation}\label{eq:dzdl}
    \frac{\partial^2 \wbar}{\partial\lapelem{i'j'}\partial\lapelem{k'l'}} = \wbar\left(\cachedelem{i'j'}\cachedelem{k'l'}-\cachedelem{i'l'}\cachedelem{k'j'}\right)
\end{equation}

\noindent Note that \cref{eq:wbar-hes} also contains a term with $\nabla^2\lap$ as it is derived from the product rule.
Since $\lap$ is a linear construction, its second derivative is zero and so we can drop this term.

\begin{figure}[t!]
\begin{algorithmic}[1]

\Func{$\algCall{\algFirst}{\funcSig{w}{\edges}{\real}, \funcSig{r}{\edges}{\real^\nR}}$}
\LinesComment{Compute first-order total; requires $\bigo{\nN^3\nRs}$ time, $\bigo{\nN^2\!+\!\nR}$ space.
}

\State Compute all $\totalw{ij}$ via \cref{eq:total} and \cref{eq:dz} in $\bigo{\nN^3}$
\State $\rbar \gets {\displaystyle\smashoperator{\sum_{\edgeij\in\edges}}}
     \totalw{ij} \redge{ij}$ {\color{gray}\algorithmiccomment{ $\bigo{\nN^2\nRs}$ }} \label{line:rbar}
\State \Return $\rbar$
\EndFunc
\end{algorithmic}
\caption{Algorithm for first-order totals.
}
\label{fig:alg}
\end{figure}

\subsection{Complexity Analysis}\label{sec:runtime}
The efficiency of our approach is rooted in the following result from automatic differentiation, which relates the cost of gradient evaluation to the cost of function evaluation.
Given a function $f$, we denote the number of differentiable elementary operations (e.g., +, *, /, -, cos, pow) of $f$ by $\runtime{f}$.
\begin{thm}[Cheap Jacobian--vector Products]
\label{thm:jvp}
For any function $\funcSig{f}{\R^{\nK}}{\R^{\nM}}$
and any vector $v \in \R^\nM$, we can evaluate $(\nabla f(x))^\top v\in\R^\nK$ with cost satisfying the following bound via reverse-mode AD \citep[Page 44]{griewank-walther},
\begin{equation}
    \runtime{(\nabla f(x))^\top v} \leq 4 \pwdot \runtime{f}
\end{equation}
Thus, $\bigo{ \runtime{(\nabla f(x))^\top v} } = \bigo{ \runtime{f} }$.
\end{thm}

As a special (and common) case, \cref{thm:jvp} implies a \emph{cheap gradient principle}: the cost of evaluating the gradient of a function of one output ($\nM=1$) is as fast as evaluating the function itself.

\paragraph{Algorithm {\normalfont $\algFirst$}.} The cheap gradient principle tells us that $\nabla \wbar$ can be evaluated as quickly as $\wbar$ itself, and that numerically accurate procedures for $\wbar$ give rise to similarly accurate procedures for $\nabla \wbar$.
Additionally, many widely used software libraries can do this work for us, such as JAX, PyTorch, and TensorFlow.
The runtime of evaluating $\wbar$ is dominated by evaluating the determinant of the Laplacian matrix.
Therefore, we can find both $\wbar$ and $\nabla\wbar$ in the same complexity: $\bigo{\nN^3}$.
\cref{line:rbar} of \cref{fig:alg} is a sum over $\nN^2$ scalar--vector multiplications of size $\nR$, this suggests a runtime of $\bigo{ \nN^2\nR}$.
However, in many applications, $\nR$ is a sparse function.
Therefore, we find it useful to consider the complexities of our algorithms in terms of the size $\nR$, and the maximum density $\nRs$ of each $\redge{ij}$.
We can then evaluate \cref{line:rbar} in $\bigo{\nN^2\nRs}$, leading to an overall runtime for $\algFirst$ of $\bigo{\nN^3+\nN^2\nRs}$.
The call to $\algMTT$ uses $\bigo{\nN^2}$ space to store the Laplacian matrix.
Computing the gradient of $\wbar$ similarly takes $\bigo{\nN^2}$ to store.%
Since storing $\rbar$ takes $\bigo{\nR}$ space, $\algFirst$
has a space complexity of $\bigo{\nN^2+\nR}$.

\begin{figure}[t!]
\begin{algorithmic}[1]
\Func{$\algCall{\algSecondVp}{\funcSig{w}{\edges}{\real}, \funcSig{r}{\edges}{\real^\nR}, \funcSig{s}{\edges}{\real^\nS}}$}
\LinesComment{Compute second-order total with gradient-vector products;
requires $\mathcal{O}(\nR(\nN^3 \!+\! \nN^2\nRs \!+\! \nN^2\nSs))$ time, $\bigo{\nN^2\nR \!+\! \nR\,\nS}$ space.
}
\For{$n=1 \ldots \nR$} {\color{gray}\algorithmiccomment{ $\bigo{\nR(\nN^3 + \nN^2\nRs)}$ }} \label{line:rloop}
\State Compute $\nabla\rbar_{n}$ using reverse-mode AD
\Statex \quad\quad\quad on $[\algFirst(w, r)]_n$
\EndFor
\LinesComment{Apply \cref{eq:tbar-r}; requires $\bigo{\nN^2 \nR \nSs}$}
\State \Return ${\displaystyle\smashoperator{\sum_{\edgeij\in\edges}}}
    \frac{\partial \rbar}{\partial \wedge{ij}}\wedge{ij}\sedge{ij}^{\top}$ 
     \label{line:tbar}
\EndFunc

\vspace{5pt}

\Func{$\algCall{\algSecondHes}{\funcSig{w}{\edges}{\real}, \funcSig{r}{\edges}{\real^\nR}, \funcSig{s}{\edges}{\real^\nS}}$}
\LinesComment{Compute second-order total by materializing Hessian; requires $\bigo{\nN^4\nRs\nSs}$ time, $\bigo{\nN^2 \!+\! \nR\,\nS}$ space.}
\State Compute all $\totalw{ij}$ using \cref{eq:total} and \cref{eq:dz}
\State Compute all $\totalw{ij,kl}$ using \cref{eq:total2} and \cref{eq:wbar-hes}
\LinesComment{Apply \cref{eq:tbar-Z}; requires $\bigo{\nN^4 \nRs \nSs}$}
\State \Return ${\displaystyle\smashoperator{\sum_{\edgeij\in\edges}}}\totalw{ij}\redge{ij}\sedge{ij}^\top + {\displaystyle\smashoperator{\sum_{\edgekl\in\edges}}}\totalw{ij,kl}\redge{ij}\sedge{kl}^\top$ \label{line:acum}
\EndFunc

\vspace{5pt}

\Func{$\algCall{\algSecond}{\funcSig{w}{\edges}{\real}, \funcSig{r}{\edges}{\real^\nR}, \funcSig{s}{\edges}{\real^\nS}}$}
\LinesComment{Unified algorithm for computing second-order total;
requires 
$\bigo{\nN^3(\nRs \!+\! \nSs) \!+\! \nR\,\nS \!+\! \nN^2\bar{\nR}\,\bar{\nS}}$ time,
$\bigo{\nR\,\nS \!+\! \nN^2(\bar{\nR} \!+\! \bar{\nS})}$ space
}
\LinesComment{The following quantities are computed in $\bigo{\nN^3}$}
\State Compute all $\totalw{ij}$ using \cref{eq:total} and \cref{eq:dz}
\State Compute $\cached$ and $\lapp$ using \cref{eq:cached} and \cref{eq:dlap}
\LinesComment{$\rbar$, $\sbar$, and $\total{f}$ are first-order quantities.}
\State $\rbar \gets {\displaystyle\smashoperator{\sum_{\edgeij\in\edges}}}
     \totalw{ij} \redge{ij}$ {\color{gray}\algorithmiccomment{$\bigo{\nN^2\nRs}$}}
\State $\sbar \gets \smashoperator{\sum_{\edgeij\in\edges}}
     \totalw{ij} \sedge{ij}$ {\color{gray}\algorithmiccomment{$\bigo{\nN^2\nSs}$}}
\State $\total{f} \gets {\displaystyle\smashoperator{\sum_{\edgeij\in\edges}}}
     \totalw{ij} \redge{ij}\sedge{ij}^\top$ {\color{gray}\algorithmiccomment{$\bigo{\nN^2\nRs\nSs}$}}
\State $\widehat{r} \gets \zerovector; \widehat{s} \gets \zerovector$
\For{$i,j,k\in\nodes$} {\color{gray}\algorithmiccomment{$\bigo{\nN^3}$}}
\For{$(i',j')\in\col{i}{j}$} {\color{gray}\algorithmiccomment{$\bigo{1}$}}
    \State $\rhat{kj'}\plusequal\cachedelem{i'k}\dlap{i'j'}{ij}\wedge{ij}\redge{ij}$ {\color{gray}\algorithmiccomment{$\bigo{\nRs}$}}
    \State $\shat{j'k}\plusequal\cachedelem{i'k}\dlap{i'j'}{ij}\wedge{ij}\sedge{ij}$ {\color{gray}\algorithmiccomment{$\bigo{\nSs}$}}
\EndFor
\EndFor
\LinesComment{Apply \cref{eq:tbar-eff}; requires $\bigo{\nR\,\nS + \nN^2\bar{\nR}\,\bar{\nS}}$}
\State \Return $\total{f} + \frac{1}{\wbar}\rbar\,\sbar^{\top} - \wbar {\displaystyle\smashoperator{\sum_{j',l'\in\nodes}}} \rhat{j'l'} \shat{j'l'}^\top$
\EndFunc
\end{algorithmic}
\caption{Three algorithms for computing second-order totals.  We recommend $\algSecond$ as it achieves the best runtime in general.  The algorithms $\algSecondVp$ and $\algSecondHes$ are presented for pedagogical purposes in \cref{sec:runtime}.}
\label{fig:second}
\end{figure}

\paragraph{Algorithm {\normalfont $\algSecondVp$}.}
Second-order quantities ($\tbar$), appear to require $\nabla^2 \wbar$ and so do not directly fit the conditions of the cheap gradient principle: the Hessian ($\nabla^2 \wbar$) is the Jacobian of the gradient.
The approach of $\algSecondVp$ to work around this is to make several calls to \cref{thm:jvp} for each element of $\rbar$.
In this case, the function in question is \cref{eq:first-order}, which has output dimensionality $\nR$.
Computing $\nabla\rbar$ can thus be evaluated with $\nR$ calls to reverse-mode AD, requiring $\bigo{\nR(\nN^3+\nN^2\nRs)}$ time.
We can somewhat support fast accumulation of $\nSs$-sparse $\nS$ in the summation of $\algSecondVp$ (\cref{line:tbar}).
Unfortunately, $\frac{\partial \rbar}{\partial \wedge{ij}}$ will generally be dense, so the cost of the outer product on \cref{line:tbar} is $\bigo{\nR \nSs}$.
Thus, $\algSecondVp$
has an overall runtime of $\bigo{\nR(\nN^3+\nN^2\nRs)+\nN^2\nR\,\nSs}$.\footnote{If $\nS{<}\nR$, we can change the order of $\algSecondVp$ to compute $\tbar^{\top}$ in $\bigo{\nS(\nN^3\!+\!\nN^2\nSs)\!+\!\nN^2\nRs\nS}$.}
Additionally, $\algSecondVp$ requires $\bigo{\nN^2\nR+\nR\,\nS}$ of space because $\bigo{\nN^2\nR}$ is needed to compute and store the Jacobian of $\rbar$ and $\tbar$ has size $\bigo{\nR\,\nS}$.

\paragraph{Algorithm {\normalfont $\algSecondHes$}.}
The downside of $\algSecondVp$
is that no work is shared between the $\nR$ evaluations of the loop on \cref{line:rloop}.
For our computation of $\wbar$, it turns out that substantial work can be shared among evaluations.
Specifically, $\nabla^2\wbar$ only relies on the inverse of the Laplacian matrix, as seen in  \cref{eq:dzdl}, leading to an alternative algorithm for second-order quantities, $\algSecondHes$.
This is essentially the same observation made in \citet{druck09covariance}.
Exploiting this allows us to compute $\nabla^2 \wbar$ in $\bigo{\nN^4}$ time.
Note that this runtime is only achievable due to the sparsity of $\nabla\lap$.
The accumulation component (\cref{line:acum}) of $\algSecondHes$
can be done in $\bigo{\nN^4\nRs\nSs}$.
Considering space complexity, while not prevalent in our pseudocode, a benefit of $\algSecondHes$
is that we do not need to materialize the Hessian of $\wbar$ as it only makes use of the inverse of the Laplacian matrix.
Therefore, we only need $\bigo{\nN^2}$ space for the Laplacian inverse and $\bigo{\nR\,\nS}$ space for $\tbar$.
Consequently, the $\algSecondHes$ requires $\bigo{\nN^2\!+\!\nR\,\nS}$ space.

\paragraph{Algorithm {\normalfont $\algSecond$}.}
So far we have seen that when $\nR$ is small, that $\algSecondVp$ can be much faster than $\algSecondHes$.
On the other hand, when $\nR$ is large and $\nRs\ll\nR$, $\algSecondHes$ can be much faster than $\algSecondVp$.
Can we get the best of $\algSecondVp$ and $\algSecondHes$?  Our unified algorithm, $\algSecond$ in \cref{fig:second}, does just that.  To derive it, we refactor the bottleneck of $\algSecondHes$ using \cref{eq:wbar-hes} and the distributive property\footnotemark{}  
\begin{align}
    \sum_{\substack{\edgeij\in\edges \\ \edgekl\in\edges}}& \frac{\partial^2\wbar}{\partial\wedge{ij}\partial\wedge{kl}} \wedge{ij} \wedge{kl} \redge{ij} \sedge{kl}^{\top} \nonumber \\
    &= \frac{1}{\wbar}\rbar\, \sbar^\top - \wbar \sum_{j',l'\in\nodes} \rhat{j'l'}\,\shat{j'l'}^{\top}
\end{align}
where
\begin{align}
    \rhat{j'l'} &= \sum_{\edgekl\in\edges}\sum_{k'\in\nodes}\cachedelem{k'j'}\dlap{k'l'}{kl}\wedge{kl}\redge{kl} \\
    \shat{j'l'} &= \sum_{\edgeij\in\edges}\sum_{i'\in\nodes}\cachedelem{i'l'}\dlap{i'j'}{ij}\wedge{ij}\sedge{ij}
\end{align}
The remainder of $\tbar$ is given by
\begin{equation}
\total{f}\defeq\sum_{\edgeij\in\edges}\totalw{ij}\redge{ij}\sedge{ij}^\top
\end{equation}
Therefore, we can find $\tbar$ by
\begin{equation}\label{eq:tbar-eff}
    \tbar = \total{f} + \frac{1}{\wbar}\rbar\, \sbar^\top - \wbar \sum_{j',l'\in\nodes} \rhat{j'l'}\,\shat{j'l'}^{\top}
\end{equation}
We provide a proof in \cref{app:proof}.

\footnotetext{Refactoring sum--product expressions via the distributive property is the cornerstone of dynamic programming; similar examples in natural language processing include \citet{eisner-blatz-2007,gildea2011grammar}.}

Now, we can compute $\rbar$ and $\sbar$ using $\algFirst$
in $\bigo{\nN^3 + \nN^2(\nRs+\nSs)}$ and their outer product in $\bigo{\nR\,\nS}$.
Additionally, we can compute all $\rhat{j'l'}$ and $\shat{j'l'}$ values in $\bigo{\nN^3 \nRs}$ and $\bigo{\nN^3 \nSs}$, respectively.  
If $r$ is $\nRs$ sparse, then each $\rhat{j'l'}$ is $\bar{\nR} \,{\defeq}\, \min(\nR, \nN\, \nRs)$ sparse. We can compute the sum over all $\rhat{j'l'}\shat{j'l'}^{\top}$ in $\bigo{\nN^2\bar{\nR}\,\bar{\nS}}$ time.
Combining these runtimes, we have that $\algSecond$
runs in $\bigo{\nN^3(\nRs+\nSs) + \nR\,\nS + \nN^2\bar{\nR}\ \bar{\nS}}$.
$\algSecond$ requires a total of $\bigo{\nR\,\nS \!+\! \nN^2(\bar{\nR} \!+\! \bar{\nS})}$: $\bigo{\nR\,\nS}$ space for $\tbar$, and $\bigo{\nN^2(\bar{\nR} \!+\! \bar{\nS})}$ space for the $\rhat{}$ and $\shat{}$ values.

We return to our original question: Can we get the best of $\algSecondVp$ and $\algSecondHes$?
In the case when $\nR$ is small, $\algSecond$ matches the runtime of $\algSecondVp$.
Furthermore, in the case when $\nR$ is large and $\nRs\ll\nR$, $\algSecond$ matches the runtime of $\algSecondHes$.
Therefore, $\algSecond$ is able to achieve the best runtime regardless of the functions $r$ and $s$.

\begin{table*}[htb]
    \centering
    \begin{tabular}{lccccc}
         \multirow{2}{*}{\bf Language} & \multirow{2}{*}{\makecell{\bf Sentence \\ \bf length}} & \multirow{2}{*}{\makecell{\bf Entropy \\ \bf (nats / word)}} & \multicolumn{2}{c}{\bf Average Runtime (ms)} & \multirow{2}{*}{\bf Speed-up} \\
         &  & & $\algFirst$ {\bf (\cref{fig:alg})} & \bf Past Approach & \\ \midrule
         Finnish & $9.23$  & $0.6092$ & $0.4623$ & $1.882$  & $4.1$ \\
         English & $12.45$ & $0.8264$ & $0.5102$ & $2.778$ & $5.4$ \\
         German  & $17.56$ & $0.8933$ & $0.5583$ & $4.104$ & $7.3$ \\
         French  & $24.65$ & $0.8923$ & $0.5635$ & $5.742$ & $10.2$ \\
         Arabic  & $36.05$ & $0.7163$ & $0.6220$ & $9.368$ & $15.1$ \\
    \end{tabular}
    \caption{Average runtime of computing entropy of dependency parser output on five languages. We use the weights of the Stanford Dependency Parser \citep{qi2018universal}. The past approach is that of \citet{smith-eisner-2007}.}
    \label{tab:ent-runtime}
\end{table*}

\section{Applications and Prior Work}
\label{sec:examples}
In this section, we apply our framework to compute a number of important quantities that are used when working with probabilistic models.  
We relate our approach to existing algorithms in the literature (where applicable), and mention existing and potential applications.  Many of our quantities were covered in \citet{li-eisner-2009} for B-hypergraphs; we extend their results to spanning trees.

In most applications that involve training a probabilistic model, the edge weights in the model will be parameterized in some fashion.
Traditional approaches \cite{koo-et-al-2007, smith-smith-2007, mcdonald,druck-thesis} use log-linear parameterizations, whereas more recent work \cite{dozat, liu-lapata-2018-learning, ma-xia-2014-unsupervised} use neural-network parameterizations.
Our algorithms are agnostic as to how edges are parameterized.

\subsection{Risk}\label{sec:uas}
Risk minimization is a technique for training structured prediction models \citep{li-eisner-2009, smith-eisner-2006-minimum, stoyanov-eisner-2012-minimum}.
Risk is the expectation of a cost function $\funcSig{r}{\st}{\real}$ that measures the number of mistakes in comparison to a target tree $\tree^*$.
In the context of dependency parsing, $\rtree{\tree}$ can be the labeled or unlabeled attachment score (LAS and UAS, respectively), both of which are additively decomposable.
The unlabeled case decomposes as follows:
\begin{equation}
    \redge{ij}=
    \begin{cases}
    \frac{1}{\nN} &  \textbf{ if } \edgeij \in \tree^* \\
    0 &  \textbf{ otherwise}
    \end{cases}
\end{equation}
where $d^*$ is the gold tree and $\nN$ is the length of the sentence.
Note that the use of $\frac{1}{\nN}$ ensures that $\rtree{\tree}$ will be a score between $0$ and $1$.
We can then obtain the expected attachment score using $\algFirst$, and we can evaluate its gradient in the same run-time using reverse-mode AD or $\algSecond$.
In this case, $\treefuncSig{s}{\nS}$ is the one-hot representation of the edges; thus, we have $\nS=\nN^2$. 
However, because $s$ is $1$-sparse, we have $\nSs=1$.
Additionally, as $r$ does not depend on $w$, we do not need to add a first-order term to find the gradient.
Therefore, the runtime for the gradient is also $\bigo{\nN^3}$.

\subsection{Shannon Entropy}\label{sec:ent}
Entropy is a useful measure of uncertainty, which has been used a number of times in dependency parsing \citep{smith-eisner-2007,druck09covariance, ma-xia-2014-unsupervised} for semi-supervised learning.
\citet{smith-eisner-2007} employ entropy regularization \cite{grandvalet-bengio-2005}
to bootstrap dependency parsing.
However, they give an algorithm for the Shannon entropy,
\begin{equation}\label{eq:ent}
    \ent(\prob) \defeq \expect{d}{- \log\ptree{\tree}}
\end{equation}
that runs in $\bigo{\nN^4}$.\footnote{Their algorithm calls MTT $\nN$ times, where the $i\th$ call to MTT multiplies the set of incoming edges to $i\th$ non-root node by their $\log$ weight.}
Recall from \cref{sec:expect} that $-\log\ptree{\tree}$ is additively decomposable;  
thus, 
running $\algFirst$ with $\redge{ij} \teq \frac{1}{\nN}\log\wbar\,{-}\,\log\wedge{ij}$
computes $\ent(\prob)$ in $\bigo{\nN^3}$. \citet{martins-etal-2010-turbo}'s algorithm for computing $\ent(\prob)$ is precisely the same as ours.  However, they do not describe how to compute its gradient.
As with risk, we can find the gradient of entropy using $\algSecond$ or using reverse-mode AD.
When using $\algSecond$, since the gradient of $r$ with respect to $w$ is not $0$, we add the first-order quantity $\algCall{\algFirst}{w, \nabla r}$ as in \cref{eq:nablabar}.
For entropy, we have that $\nabla\redge{ij}=\frac{1}{\nN\wbar}\nabla\wbar-\frac{1}{\wedge{ij}}\onehot{ij}$.

\paragraph{Experiment.} We briefly demonstrate the practical speed-up over \citet{smith-eisner-2007}'s $\bigo{\nN^4}$ algorithm.
We compare the average runtime per sentence of five different UD corpora.\footnote{Times were measured using an Intel(R) Core(TM) i7-7500U processor with 16GB RAM.}
The languages have different average sentence lengths to demonstrate the extra speed-up gained when calculating the entropy of longer sentences (that is, $\st$ would be a larger set).
\cref{tab:ent-runtime} shows that even for a corpus of short sentences (Finnish), we achieve a $4$ times speed-up.
This increases to $15$ times as we move to corpora with longer sentences (Arabic).

\subsection{Kullback--Leibler Divergence}\label{sec:kl}
To the best of our knowledge, no algorithms to compute the Kullback--Leibler (KL) divergence between two graph-based parsers (nor its gradient) have been given in the literature.
We show how this can be achieved easily within our framework.
The KL divergence is defined as
\begin{equation}
    \KL(\prob \mid\mid \qprob) \defeq \sumtree\ptree{\tree}\log\frac{\ptree{\tree}}{\qtree{\tree}}
\end{equation}
This takes a similar form to the Shannon entropy in \cref{eq:ent}.
We can therefore choose our additively decomposable function to be $\redge{ij}=\log\frac{\wedge{ij}}{q_{ij}} - \frac{1}{\nN}\log\wbar$.
Running $\algFirst$ with these weights computes the KL divergence in $\bigo{\nN^3}$ time.
To find the gradient of the KL divergence, we return the sum of $\algCall{\algSecond}{w, r, s}$ where we chose $\sedge{ij}=\frac{1}{\wedge{ij}}\onehot{ij}$ and add $\algCall{\algFirst}{w, \nabla r}$.
For the KL divergence, we have that $\nabla\redge{ij}=\frac{1}{\wedge{ij}}\onehot{ij}-\nabla\wbar\frac{1}{\nN\wbar}$.

\subsection{Gradient of the GE Objective}\label{sec:ge}
The generalized expectation criterion \citep{mccallum2007generalized,druck-etal-2009-semi} 
is a method for semi-supervised training using weakly labeled data. 
GE fits model parameters
by encouraging models to match certain expectation constraints,
such as marginal-label distributions, on the unlabeled data.
More formally, let $f$ be a feature function $f(\tree)\in\real^\nF$, 
and with a target value of $\target\in\real^{\nF}$ that has been specified using domain knowledge.  
For example, given an English part-of-speech tagged sentence, 
we can provide the following light supervision to our model: determiners should attach to the nearest noun on their right.  
This is an example of a very precise heuristic for dependency parsing English that has high precision.

GE then minimizes the following objective,
\begin{align}\label{eq:ge}
&\mathrm{GE}(p, \target) = \frac{1}{2}\Big|\Big|\expect{\tree}{f(\tree)} - \target \Big|\Big|^2
\end{align}
which encourages the model parameters to match the target expectations.
Most methods for optimizing \cref{eq:ge} will make use of the gradient.

We note that by application of the chain rule, the gradient of the GE objective is a second-order quantity, and so we can use $\algSecond$
to compute it.
As we discussed in \cref{cautionary-tales}, the gradient of the GE has led to confusion in the literature \citep{druck-etal-2009-semi, druck09covariance, druck-thesis}.
The best runtime bound prior to our work is \citet{druck-etal-2009-semi}'s $\bigo{\nN^4 \nFs}$ algorithm.  $\algSecond$ is strictly better at $\bigo{\nN^3 {+} \nN^2 \nFs}$ time.\footnote{We must apply a chain rule in order to use $\algSecond$. To do this, we first run $\algFirst$ to obtain $\total{f}$ in $\bigo{\nN^3+\nN^2\nFs}$. We then run $\algSecond$ with the dot product of $f$ and $\total{f}-\target$, which has a dimensionality of $1$, and the sparse one-hot vectors as before. The execution of $\algSecond$ then takes $\bigo{\nN^3}$, giving us the desired runtime. Full detail is available in our code.}
Alternatively, as the GE objective is a scalar, we can compute its gradient in $\bigo{\nN^3\!+\!\nN^2 \nFs}$ using reverse-mode AD.
\citet{druck-thesis} acknowledges that AD can be used, but questions its practicality and numerical accuracy.  We hope to dispel this misconception in the following experiment.

\paragraph{Experiment.} We compute the GE objective and its gradient for almost $1500$ sentences of the English UD Treebank\footnote{We used all sentences in the test set, which were between five and 150 words.} \citep{ud} using $20$ features extracted using the methodology of \citet{druck-etal-2009-semi}.
We note that $\algSecond$ obtains a speed-up of $9$ times over \citet{druck09covariance}'s strategy of materializing the covariance matrix (i.e., $\algSecondHes$).
Additionally, the gradients from both approaches are equivalent with an absolute tolerance of $10^{-16}$.

\section{Conclusion}\label{sec:conc}
We presented a general framework for computing first- and second-order expectations for additively decomposable functions.
We did this by exploiting a key connection between gradients and expectations that allows us to solve our problems using automatic differentiation.
The algorithms we provide are simple, efficient, and extendable to many expectations.
The automatic differentiation principle has been applied in other settings, such as weighted context-free grammars \citep{eisner16backprop} and chain-structured models \citep{vieira-etal-2016-speed}.
We hope that this paper will also serve as a tutorial on how to compute expectations over trees so that the list of \emph{cautionary tales} does not grow further.
Particularly, we hope that this will allow for the KL divergence to be used in semi-supervised training of dependency parsers.
Our aim is for our approach for computing expectations to be extended to other structured prediction models.

\section*{Acknowledgments}
We would like to thank action editor Dan Gildea and the three anonymous reviewers for their valuable feedback and suggestions. 
The first author is supported by the University of Cambridge School of Technology Vice-Chancellor's Scholarship as well as by the University of Cambridge Department of Computer Science and Technology's EPSRC.

\bibliography{tacl2020}
\bibliographystyle{acl_natbib}

\clearpage
\appendix

\section{Derivation of $\nabla^2\wbar$}\label{app:hes}
In this section, we will provide a derivation for the expression of $\nabla^2\wbar$ given in \cref{eq:wbar-hes}.
We begin by taking the derivative of $\nabla\wbar$ using \cref{eq:dz}
\begin{align*}
    \frac{\partial^2\wbar}{\partial\wedge{ij}\partial\wedge{kl}} = \frac{\partial}{\partial\wedge{ij}} \wbar\  \smashoperator{\sum_{(k',l')\in\col{k}{l}}}\cachedelem{k'l'}\dlap{k'l'}{kl}
\end{align*}
We solve this by applying the product rule.\footnote{Note that we do not have to take the derivative of $\dlap{k'l'}{kl}$ as it is either $1$ or $-1$.}
The first term of the product rule is
\begin{align*}
    &\frac{\partial\wbar}{\partial\wedge{ij}}\ \ \ \smashoperator{\sum_{(k',l')\in\col{k}{l}}}\cachedelem{k'l'}\dlap{k'l'}{kl} \\
    &=\wbar\smashoperator{\sum_{\substack{(i',j')\in\col{i}{j} \\(k',l')\in\col{k}{l}}}}\cachedelem{i'j'}\cachedelem{k'l'}\dlap{i'j'}{ij}\dlap{k'l'}{kl}
\end{align*}
The second term of the product rule is
\begin{align*}
    &\wbar\smashoperator{\sum_{(k',l')\in\col{k}{l}}}\ \ \frac{\partial\cachedelem{k'l'}}{\partial\wedge{ij}}\dlap{k'l'}{kl} \\
    &= -\wbar\smashoperator{\sum_{\substack{(i',j')\in\col{i}{j} \\(k',l')\in\col{k}{l}}}}\ \ \cachedelem{i'l'}\cachedelem{k'j'}\dlap{i'j'}{ij}\dlap{k'l'}{kl}
\end{align*}
Summing these together yields \cref{eq:wbar-hes}.

\section{Proof of {\normalfont $\algSecond$}}\label{app:proof}
In this section, we will prove the decomposition of $\tbar$ that allows for the efficient factoring used in $\algSecond$.
First, recall from \cref{prop:tbar} that we may find $\tbar$ by
\begin{align*}
   \tbar &= \sum_{\edgeij\in\edges}\left[\frac{\partial\wbar}{\partial \wedge{ij}}\wedge{ij}\redge{ij}\sedge{ij}^{\top}\right] +  \\
   & \hphantom{=}\sum_{\edgeij\in\edges}\sum_{\edgekl\in\edges}\left[\frac{\partial^2\wbar}{\partial \wedge{ij}\partial \wedge{kl}}\wedge{ij}\wedge{kl}\redge{ij}\sedge{kl}^{\top} \right]
\end{align*}
The first summand is the first-order total for function $\redge{ij}\sedge{ij}^\top$ (given as $\total{f}$ in $\algSecond$).
We can write a sum over all edges as the sum over pairs of nodes in $\nodes$.
Similarly, elements in $\col{i}{j}$ can be considered as pairs of nodes.
Therefore, unless specified otherwise, we assume all variables in the base of a summation are scoped to $\nodes$.
Then, the second summand can then be rewritten
\begin{align*}
    &\sum_{\edgeij\in\edges}\sum_{\edgekl\in\edges}\frac{\partial^2\wbar}{\partial \wedge{ij}\partial \wedge{kl}}\wedge{ij}\wedge{kl}\redge{ij}\sedge{kl}^{\top}  \\
    &= \ \monstersum{i,j,k,l,i',j',k',l'}\dlap{i'j'}{ij}\wbar\cachedelem{i'j'}\cachedelem{k'l'}\dlap{k'l'}{kl}\wedge{ij}\wedge{kl}\redge{ij}\sedge{kl}^{\top} \\
    & \hphantom{\ \monstersum{i,j,k,l}}  - \dlap{i'j'}{ij}\wbar\cachedelem{i'l'}\cachedelem{k'j'}\dlap{k'l'}{kl}\wedge{ij}\wedge{kl}\redge{ij}\sedge{kl}^{\top}
\end{align*}
By distributivity, the first term equals
\begin{align*}
    & \wbar\,\Bigg[ \monstersum{\ \ \ \ \ i,j,i',j'}\cachedelem{i'j'}\dlap{i'j'}{ij}\wedge{ij}\redge{ij} \! \Bigg] \!\! \Bigg[ \monstersum{\ \ \ \ \ k,l,k',l'}\cachedelem{k'l'}\dlap{k'l'}{kl}\wedge{kl}\sedge{kl} \! \Bigg]^\top \\
    &= \frac{1}{\wbar} \rbar\,\sbar^{\top}
\end{align*}
By distributivity, the second term equals
\begin{align*}
    &
    \wbar \monstersum{j',l'}
      \Bigg[\underbrace{\monstersum{k',k,l}
        \cachedelem{k'j'}\dlap{k'l'}{kl}\wedge{kl}\redge{kl}
      }_{\defeq \rhat{j'l'}} \Bigg] \\
     &\hphantom{\wbar\monstersum{j',l'}\monstersum{k',k,l}\cachedelem{k'j'}\dlap{kl}{kl}}
      \Bigg[\underbrace{\monstersum{i',i,j}\cachedelem{i',l'}\dlap{i'j'}{ij}\wedge{ij}\sedge{ij} }_{\defeq \shat{j'l'}} \Bigg]^{\top}
     \\
    &= \wbar\sum_{j',l'}\rhat{j'l'}\,\shat{j'l'}^{\top}
\end{align*}
The above decomposition assumed we sum over all $i'$, $j'$, $k'$, and $l'$ and so suggests we can compute all $\rhat{j'l'}$ and $\shat{j'l'}$ in $\bigo{\nN^5(\nRs + \nSs)}$.
However, we can exploit the sparsity of $\nabla\lap$ to improve this.
Specifically, the follow algorithm computes $\rhat{j'l'}$ for all $j',l'\in\nodes$.
\begin{algorithmic}
\State $\rhat{j'l'} \gets \boldsymbol{0}$
\For{$\edge{k}{l} \in \edges$}  \Comment{$\bigo{\nN^2}$}
\For{$\edge{k'}{l'} \in \col{k}{l}$}  \Comment{$\bigo{1}$}
\For{$j'\in\nodes$}  \Comment{$\bigo{\nN}$}
    \State $\rhat{j'l'} \pluseq
    \cachedelem{k'l'}\dlap{k'l'}{kl}\wedge{kl}\redge{kl}$
\EndFor
\EndFor
\EndFor
\end{algorithmic}
Therefore, we can compute all $\rhat{j'l'}$ and $\shat{j'l'}$ in $\bigo{\nN^3(\nRs + \nSs)}$.
Each $\rhat{ij}$ is at most $\bigo{\nN\nRs}$ dense, 
because there are at most $\bigo{\nN}$ $\nRs$-sparse vectors added to it (by the inner loop).
Hence, $\rhat{ij}$ is $\bigo{\bar{\nR}}$ sparse where $\bar{\nR} {\defeq} \min(\nR, \nN\,\nRs)$.
This means that computing the sum of the outer-products of all $\rhat{ij}$ and $\shat{ij}$ can be done in
$\bigo{\nN^2\bar{\nR}\,\bar{\nS}}$.
Then, given that we have
\begin{equation*}
\tbar = \total{f} + \frac{1}{\wbar}\rbar\,\sbar - \wbar\sum_{j',l'}\rhat{j'l'}\,\shat{j'l'}^{\top}
\end{equation*}
We can find $\tbar$ in 
$$\mathcal{O}\big(\nN^3(\nRs{+}\nSs){+}\nR\,\nS{+}\nN^2\bar{\nR}\,\bar{\nS}\big)$$

\end{document}